\documentclass[journal,comsoc]{IEEEtran}

\usepackage[T1]{fontenc}

\usepackage[utf8]{inputenc}
\usepackage[english]{babel}
\usepackage{amsmath, amssymb, amsfonts, amsthm}
\usepackage{graphicx}
\usepackage[caption=false,font=footnotesize]{subfig}
\usepackage{enumitem}
\usepackage[linkcolor=blue,colorlinks=true,citecolor=blue]{hyperref}

\newcommand{\vect}[1]{\mathbf{#1}}
\newcommand{\vertiii}[1]{{\vert\kern-0.25ex\vert\kern-0.25ex\vert #1 
    \vert\kern-0.25ex\vert\kern-0.25ex\vert}}

\newtheorem{theorem}{Theorem}

\newtheorem{lemma}{Lemma}

\newtheorem{remark}{Remark}
\newtheorem{assumption}{Assumption}

\begin{document}
\title{Structure Learning of Sparse GGMs over Multiple Access  Networks}

\author{Mostafa~Tavassolipour,
        Armin~Karamzade,
        Reza~Mirzaeifard,
        Seyed~Abolfazl~Motahari,
        ~and~Mohammad-Taghi~Manzuri~Shalmani}

%

\maketitle

\begin{abstract}
A central machine is interested in estimating the underlying structure of a sparse Gaussian Graphical Model (GGM) from datasets distributed across multiple local machines. The local machines can communicate with the central machine through a wireless multiple access channel. In this paper, we are interested in designing effective strategies where reliable learning is feasible under power and bandwidth limitations. Two approaches are proposed: Signs and Uncoded methods. In Signs method, the local machines quantize their data into binary vectors and an optimal channel coding scheme is used to reliably send the vectors to the central machine where the structure is learned from the received data. In Uncoded method, data symbols are scaled and transmitted through the channel. The central machine uses the received noisy symbols to recover the structure. Theoretical results show that both methods can recover the structure with high probability for large enough sample size. Experimental results indicate the superiority of Signs method over Uncoded method under several circumstances. 

\end{abstract}


\IEEEpeerreviewmaketitle

\section{Introduction} \label{sec:introduction}
\IEEEPARstart{I}{n} recent years, by the explosion of the volume of training data, distributed machine learning has become more important than ever. Many modern big datasets are distributed over several hosting machines which are connected to each other via some communication links. In such systems, designing distributed learning algorithms which efficiently exploit the available resource demands a careful design where intensive computational workloads and the amount of data communications are taken into account.

This paper is focused on the problem of structure learning in Gaussian Graphical Models (GGMs) in distributed environments. A GGM for a $d$-dimensional random vector $\vect x = (x_1, \cdots, x_d)^T \in \mathbb{R}^d$ is specified by a graph $\mathcal{G}(\mathcal{V}, \mathcal{E})$ where $\mathcal{V}=\{1,\cdots, d\}$ is the set of vertices and $\mathcal{E} \subseteq \mathcal{V}^2$ is the set of edges. The model comprises all $d$-dimensional normal distributions $\mathcal{N}(\mu, \Theta^{-1})$ where $\Theta$ is the precision  matrix with $\Theta_{jk} \neq 0$ iff $(j, k) \in \mathcal{E}$. It worths mentioning that GGM is indeed a Markov Random Field (MRF). 

In our problem setting, we assume that the data are distributed over multiple local machines so that each one contains a single dimension of the whole dataset. The local machines are connected to a central machine via a wireless medium. The central machine is responsible for inferring the conditional dependencies between the gathered data by the local machines. The system's block diagram is depicted in \figurename{}~\ref{fig:whole system diagram}. We also assume that the communication links between the local and the central machine are bandwidth limited implying that transmission of the whole local datasets to the central machine is impossible. Due to this constraint, each local machine transmits some information from its local dataset to the central machine. Then, the central machine estimates the underlying graph structure using received information from the local machines.

\begin{figure*}[t]
	\centering
	\includegraphics[scale=.7]{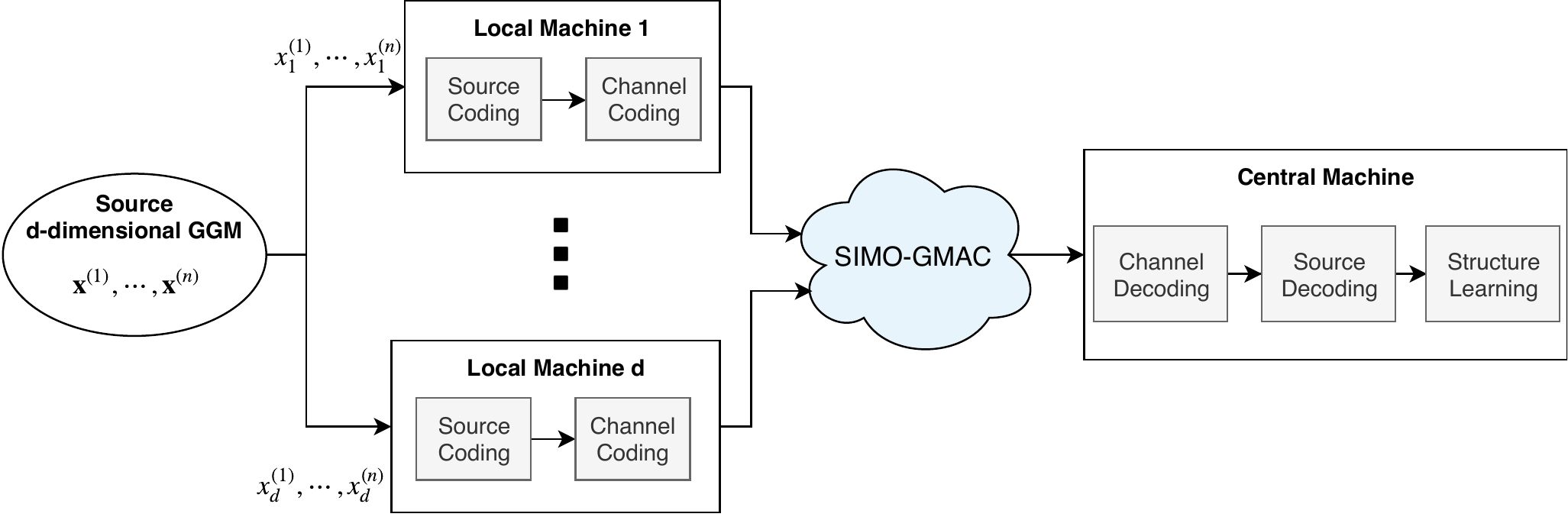}
	\caption{The whole system block diagram.}
	\label{fig:whole system diagram}
\end{figure*}

In this paper, we consider a wireless multiple access channel between the local machines and the central node. Each local machine is equipped with a single antenna while the central node is equipped with multiple antennas. Hence, the overall channel between the local and central machines is modeled as a single-input multiple-output multiple access channel (SIMO-MAC).

We have proposed two communication schemes for transmitting information from the local machines to the central machine. In the first scheme, we have separate the source coding from the channel coding. In this scheme, we quantize the source samples into single bits and assume that there exists a channel coding such that the bits can be sent through the channel reliably. We refer to this scheme as \emph{Signs method} in the paper.

In the second scheme, we do not use any source and channel coding. In this scheme, we put the source samples into the channel without any encoding. At the central machine, we directly estimate the underlying graph structure using received data from the channel. We refer to this scheme as \emph{Uncoded method}.

We have shown through theoretical analysis and experiments that by transmitting only 1 bit per sample; the central machine can reliably recover the underlying graph structure. More precisely, we have shown theoretically that by consuming only 1 bit per sample, under some mild conditions, the central machine can perfectly recover the graph structure with high probability. Moreover, the true signs of the edges weights of the graph are obtained.

The paper is organized as follows. In Section \ref{sec:related work}, we provide a brief review on structure learning of GGMs. Section \ref{sec:problem definition} describes the detail of our modeling for the source and the communication channel. In sections \ref{sec:signs method} and \ref{sec:uncoded method}, we describe Signs and Uncoded methods, respectively. Section \ref{sec:experiment} provides experimental results to compare and evaluate the proposed methods. Finally, Section \ref{sec:conclusion} concludes the paper.

\section{Related Work} \label{sec:related work}
The problem of structure learning of GGMs from data samples has applications in many fields including biology and social networks. For example, it has been used for gene regulatory networks reconstruction  in \cite{chai2014review, hecker2009gene} and  analysis of users relationships in social networks in \cite{xiang2010modeling}.  There are many studies addressing this problem from various perspectives \cite{ravikumar2011high, sojoudi2016equivalence,banerjee2008model}. 

The  Chow-Liu algorithm obtains the maximum likelihood estimate of the structure if the underlying graph is a tree \cite{chow1968approximating}. Although this algorithm is applicable for discrete random variables, it  can be used for tree structured GGMs in a similar manner \cite{anandkumar2012high}. Tavassolipour et al. \cite{tavassolipour2018learning} proposed a distributed version of the Chow-Liu algorithm and proved it can recover the underlying tree structure with high probability. Tan et al. in \cite{tan2011learning} and \cite{tan2011large} provided an analysis of the error exponent of the Chow-Liu algorithm on tree-structured GGMs.

GGMs have the property that the neighbors of each variable can be obtained by solving a linear regression problem for the corresponding variable on other variables. This approach is referred to as \emph{neighborhood selection} in the literature. For the sparse structures, there are some methods which penalize the linear regression problem with $\ell_1$ of the coefficient vector \cite{meinshausen2006high, chen2014selection}. 

Among the proposed methods for sparse structure estimation of GGMs, the $\ell_1$-regularized maximum likelihood approaches are more popular \cite{banerjee2008model,meinshausen2006high,hsieh2014quic}. This class of model is analyzed in several articles (see \cite{ravikumar2011high} and refs therein). For instance, Ravikumar et al. \cite{ravikumar2011high} analyzed the performance of the $\ell_1$-regularized maximum likelihood estimator (MLE) under high dimensional scaling. They showed that with probability converging to one, the estimated structure correctly specifies the zero pattern of the true precision matrix. A similar study is conducted by \cite{rothman2008sparse} which analyzed the consistency of the $\ell_1$-regularized MLE in the Frobenius norm.

Besides the lasso typed estimators, thresholding based estimators are proposed  for sparse recovering of the precision matrix \cite{sojoudi2016equivalence, anandkumar2012high}. For example, Sojoudi \cite{sojoudi2016equivalence} proposed a simple thresholding method and showed, under certain conditions, the resulting structure is identical to the structure obtained by the lasso.

In the distributed setting, there are several studies which address the problem of covariance/precision matrix estimation \cite{meng2013distributed, wiesel2012distributed, arroyo2016efficient}. Arroyo and Hou \cite{arroyo2016efficient} studied the problem of sparse precision matrix estimation in the situation where the samples are distributed among several machines. Their work differs from our setting in the sense that we assume the data are split across dimensions whereas they split the data across samples. Meng et al. \cite{meng2013distributed} addressed estimation of the precision matrix in a distributed manner where the zero pattern of the precision matrix is known in advanced.

\section{Problem Formulation} \label{sec:problem definition}
We are given $n$ i.i.d. random vectors drawn from a $d$-dimensional zero mean normal distribution $\mathcal{N}(\vect 0, Q_x)$ with $(Q_x)_{jj} = 1$. The focus of this paper is the problem of estimating the zero pattern of the sparse precision matrix $\Theta_x = Q_x^{-1}$ in a situation where the data is stored in $d$ separate local machines such that each machine possesses one dimension of the sample vectors.


Denoting the whole gathered data by $\{\vect x^{(1)}, \cdots, \vect x^{(n)}\}$, the $j$th machine captures the $j$-th dimension of the sample vectors. We denote the $j$-th dimension of the $i$-th sample by $x_j^{(i)}$, i.e. $\vect x^{(i)} = [x^{(i)}_1, \cdots, x^{(i)}_d]^T$. Hence, the local data at the $j$-th local machine is $\{x_j^{(1)}, \cdots, x_j^{(n)} \}$.


The probability density function of the normal distribution is given by
\begin{equation}
f(\vect x; \Theta) = \frac{1}{\sqrt{\det (2\pi\Theta^{-1})}} \exp\left\{-\frac{1}{2}\vect x^T \Theta \vect x \right\}.
\end{equation}
The negative log-likelihood of $n$ i.i.d. samples is given by
\begin{equation} \label{eq:likelihood function}
    g(\Theta) = \mathrm{tr}(\Theta S_x) - \log\det(\Theta),
\end{equation}
where  $S_x = \frac{1}{n} \sum_{i=1}^{n} \vect{x}^{(i)} \left(\vect x^{(i)}\right)^T$ is reffered to as sample covariance matrix. One of the well known methods for estimating the sparse precision matrix is to solve an $\ell_1$-regularized maximum likelihood function stated as
\begin{equation} \label{eq:ML problem}
\widehat\Theta = \arg\min_{\Theta \succ 0} \left\{g(\Theta) + \lambda_n \Vert \Theta \Vert_{1, \mathrm{off}}\right\} ,
\end{equation}
where $\Vert \Theta \Vert_{1, \mathrm{off}} := \sum_{j \neq k} \vert \Theta_{jk} \vert$, and $\lambda_n$ is a user defined regularization parameter. There is an efficient algorithm known as \emph{glasso} \cite{friedman2008sparse} for solving the above log-determinant program.

In \cite{ravikumar2011high}, it is stated that the Hessian matrix of  \eqref{eq:likelihood function} is given by
\begin{equation}
\Gamma = \nabla_{\Theta}^2 g(\Theta) \Big|_{\Theta = \Theta_x} = \Theta_x^{-1} \otimes \Theta_x^{-1}.
\end{equation}
where $\otimes$ is the Kronecker matrix product. The entry $\Gamma_{(j,k),(l,m)}$ corresponds to the second derivative $\frac{\partial^2 g}{\partial\Theta_{jk}\partial\Theta_{lm}}$, evaluated at $\Theta = \Theta_x$ and $\Gamma_{(j,k),(l,m)} = \mathrm{cov}\{x_j x_k, x_l x_m\}$.

We define the set of non-zero entries in the true precision matrix $\Theta_x$ as
\begin{equation}
S(\Theta_x) = \{(j,k) \in \mathcal{V}^2 | (\Theta_x)_{jk} \neq 0 \}.
\end{equation}
Recall that $\mathcal{V} = \{1, \cdots, d\}$. Note that $S(\Theta_x)$ includes the diagonal entries. Let $S^c(\Theta_x)$ be the complement set of $S(\Theta_x)$ which includes all pairs $(j,k)$ where  $(\Theta_x)_{jk} = 0$. For any two subsets $T$ and $T'$ of $V \times V$, the notation $\Gamma_{TT'}$ denotes a $|T|\times |T'|$ sub-matrix of $\Gamma$ with rows and columns indexed by $T$ and $T'$, respectively.

We adopt the incoherence condition used by Ravikumar et al. \cite{ravikumar2011high} to obtain error bounds on consistency of the solutions. We define max-row-sum norm of a $d$ by $d$ matrix $X$ as
\begin{equation}
    \vertiii{X}_{\infty} = \max_{i = 1,\cdots,d} \sum_{j = 1}^{d} \left\vert X_{ij} \right\vert.
\end{equation}

\begin{assumption} \label{assump: incoherence}
	(\textbf{Incoherence Condition \cite{ravikumar2011high}}) \\
	There exists some $\alpha \in (0, 1]$ such that
	\begin{equation}
         \vertiii{\Gamma_{S^c S} (\Gamma_{SS})^{-1}}_{\infty} \leq (1-\alpha).
	\end{equation}
\end{assumption}
An implication of the above condition is that the non-edge pairs cannot have strong influence on the edges. 

\begin{assumption} \label{assump:covariance control}
	(\textbf{Covariance Control \cite{ravikumar2011high}}) \\
	There exist constants $\kappa_{\Sigma},  \kappa_{\Gamma} < \infty$ such that
	\begin{align}
	    \vertiii{Q_x}_{\infty} &\leq \kappa_{\Sigma}, \\
	    \vertiii{\Gamma_{SS}^{-1}}_{\infty} &\leq \kappa_{\Gamma}.
	\end{align}
\end{assumption}
The above assumptions imply that the covariance elements along any row of $Q_x$ and $\Gamma_{SS}^{-1}$ have bounded $\ell_1$ norms.

Having no access to the original data,  the central machine finds the solution of the optimization problem \eqref{eq:ML problem} using received data from the local machines. The likelihood function $g(\Theta)$ in \eqref{eq:likelihood function} depends on the samples via the sample covariance matrix $S_x$. Thus, obtaining an appropriate approximate of the sample covariance matrix at the central machine would result in a good estimate of the underlying structure. 


\subsection{Channel Model} \label{sec:channel model}
There is a wireless channel between the local machines and the central node. In this paper, we assume that the channel can be modeled as a single-input multiple-output Gaussian multiple access channel (SIMO-GMAC) with additive white noise. The number of antennas at the receiver is assumed to be $m$. All transmitters have equal transmit power which is denoted by $p$. This implies, for all $j$, the following constraint on the transmit symbols should be satisfied:
\begin{equation}
\frac{1}{n} \sum_{i=1}^{n} \left|s_j^{(i)}\right|^2 \leq p,
\end{equation}
where $s_j^{(i)}$'s are the channel inputs at the local machine $j$.
Denoting the transmit symbols by vector $\vect s$, the channel output is modeled by 
\begin{equation} \label{eq:y=Hx+z}
\vect y =  H \vect s + \vect z,
\end{equation}
where $H\in \mathcal{C}^{m\times d}$ is assumed to be an invertible complex matrix. In the fading environments, the channel gains are drawn from independent circularly symmetric complex Gaussian distribution. The additive noise $\vect z$ is an independent circularly symmetric Gaussian vector with covaraience matrix  $\sigma_z^2 I_d$.

\section{Signs Method} \label{sec:signs method}
We assume that each local machine applies a sign function on its local dataset to obtain binary data. More precisely, given samples $\{x_j^{(1)}, \cdots, x_j^{(n)}\}$ at local machine $j$, it obtains the signs dataset $\{\hat x_j^{(1)}, \cdots, \hat x_j^{(n)}\}$ where $\hat{x}_j = \mathrm{sign}(x_j)$. Then, it transmits the binary data to the central machine with the rate of $1$ bit per sample using a channel encoder and decoder.

Denoting the bit rate of the channel at local machine $j$ by $R_j$, the achievable bit rates for all machines are characterized by \cite{tse2005fundamentals}
\begin{equation} \label{eq:channel rates region}
\sum_{k\in S} R_k \leq \lg \det \left(\frac{p}{\sigma_{z}^2} H^H_S H_S + I_{|S|} \right), \qquad \forall~ S \subseteq \{1, \cdots, d\},
\end{equation}
where  $H_S$ is the sub-matrix of $H$ that includes the rows and columns indexed by $S$, $H^H$ is the Hermitian of $H$, $I_{|S|}$ is the $|S|\times|S|$ identity matrix, and $\lg(\cdot)$ is the logarithm function in base 2. Throughout this section, we assume that $R_j \geq 1$, for all $j = 1, \cdots, d$. Therefore, there exists a channel encoding to transmit 1 bit per sample at each local machine.


At the central machine, our goal is to solve the optimization problem \eqref{eq:ML problem} on the received binary data from all local machines. In order to obtain a solution that is close to the solution obtained by the original data, we seek a suitable  approximation for the sample covariance matrix $S_x$ in \eqref{eq:ML problem}. Thus, our goal is to estimate the sample covariance matrix as accurate as possible using the received signs data. In this section, we propose an estimator for the sample covariance matrix and theoretically show its error decreases exponentially by the sample size $n$.

Let $x_j$ and $x_k$ be jointly normal with zero means, unit variances and correlation coefficient $\rho_{jk}$. If $\hat x_j$ and $\hat x_k$ be the corresponding sign variables, then the joint probability mass function (pmf) of $\hat x_j$ and $\hat x_k$ is given by \cite{bacon1963approximations}
\begin{equation} \label{eq:binary pmf}
\begin{array}{c|cc}
\hat x_j \backslash \hat x_k & -1 & +1	\\
\hline
-1 & 
\beta_{jk}/2  & (1-\beta_{jk})/2	\\
+1 & (1-\beta_{jk})/2 & \beta_{jk}/2	
\end{array}
\end{equation}
where $\beta_{jk} \in [0, 1]$ and given by
\begin{equation} \label{eq:theta}
\beta_{jk} = \frac{1}{2} + \frac{\arcsin(\rho_{jk})}{\pi}.
\end{equation}
The equation \eqref{eq:theta} can be rewritten as
\begin{equation} \label{eq:rho}
\rho_{jk} = \sin(\pi(\beta_{jk} - \frac{1}{2})) = -\cos(\pi\beta_{jk}).
\end{equation}
Thus, by proposing an estimator for $\beta_{jk}$, using \eqref{eq:rho} we can obtain an estimator for $\rho_{jk}$. The following estimator for $\beta_{jk}$ is optimal in the sense that it is unbiased and has minimum variance (UMVE) \cite{el2017rate},
\begin{equation}\label{eq:theta estimator}
\hat{\beta}_{jk} = \frac{1}{n} \sum_{i=1}^{n} \mathbb{I}(\hat x^{(i)}_j \hat x^{(i)}_k = 1),
\end{equation} 
where $\mathbb{I}(.)$ is the indicator function. $\hat\beta_{jk}$ is indeed a binomial random variable with success probability of $\beta_{jk}$. By substituting  $\hat\beta_{jk}$ in \eqref{eq:rho}, we use
\begin{equation}\label{eq:rho hat}
\hat \rho_{jk} = -\cos(\pi\hat\beta_{jk}),
\end{equation}
as an estimator for $\rho_{jk}$. Although $\hat\rho_{jk}$ is indeed biased, it is a consistent estimator for $\rho_{jk}$. Following lemma gives an error bound on this estimator.

\begin{lemma} \label{lemma:binary error bound}
	Let $x_j$ and $x_k$ be two jointly normal variables with zero means, unit variances and correlation coefficient $\rho_{jk}$. Then, for the estimator \eqref{eq:rho hat}, we have
	\begin{equation} \label{eq:binary error bound}
    \Pr\left(\left\vert\hat\rho_{jk} - \rho_{jk}\right\vert \geq \delta\right) \leq 2 \exp\left(-\frac{2}{\pi^2}n\delta^2\right),
	\end{equation}
	where $\delta \geq 0$.
\end{lemma}
\begin{proof}
	Since the function $\cos(\cdot)$ is a 1-Lipschitz function, i.e. $\vert \cos(x) - \cos(y) \vert \leq \vert x - y \vert$, we have
	\begin{align*}
	\Pr\left(\vert\hat\rho_{jk} - \rho_{jk} \vert \geq \delta\right) &= \Pr\left(\vert -\cos(\pi\hat\beta_{jk}) + \cos(\pi\beta_{jk}) \vert \geq \delta\right) \\
	&\leq \Pr\left(\pi\vert \hat\beta_{jk} - \beta_{jk} \vert \geq \delta\right) \\
	&= \Pr\left(\vert \hat\beta_{jk} - \beta_{jk} \vert \geq \frac{\delta}{\pi}\right).
	\end{align*}
	Since $\hat\beta_{jk}$ is sum of $n$ independent Bernoulli random variables, applying the Hoeffding inequality yields
	\begin{align*}
    \Pr\left(\vert \hat\beta_{jk} - \beta_{jk} \vert \geq \frac{\delta}{\pi}\right) &\leq 2\exp\left(-\frac{2}{\pi^2}n\delta^2\right),
	\end{align*}
	which completes the proof.
\end{proof}
Lemma \ref{lemma:binary error bound} shows that the error of proposed estimator $\hat \rho_{jk}$ is controlled by the number of samples exponentially. Using this estimator, we can obtain an estimator for the sample covariance matrix as follows
\begin{equation} \label{eq:sample cov matrix binary}
\widehat{S}_x = - \cos(\frac{\pi}{2} B),
\end{equation} 
where $B_{jk} = \hat\beta_{jk}$, and $\cos(.)$ function is applied on the input matrix element wise, i.e. $(\widehat{S}_x)_{ij} = -\cos(\dfrac{\pi}{2} \beta_{ij})$.
By substituting the above sample covariance matrix into \eqref{eq:ML problem}, we can solve the regularized maximum likelihood problem. 

Note that the sample covariance matrix $\widehat{S}_x$ defined in \eqref{eq:sample cov matrix binary} is not necessarily positive semi-definite. But, it does not affect the convexity and uniqueness solution of \eqref{eq:ML problem}. Ravikumar et al. \cite{ravikumar2011high} proved that the problem \eqref{eq:ML problem} is convex and has a unique solution for any sample covariance matrix with strictly positive diagonal elements which holds for $\widehat{S}_x$ in \eqref{eq:sample cov matrix binary}.


By incorporating Lemma \ref{lemma:binary error bound} and the theorems 1 and 2 in \cite{ravikumar2011high}, we can conclude that the precision matrix obtained by solving \eqref{eq:ML problem} with the sample covariance matrix $\widehat{S}_x$ in \eqref{eq:S_x tilde}, recovers the true structure with high probability. Moreover, the proposed method correctly recovers the signs of the edges with high probability.

More precisely, the event $\mathcal{M}(\Theta_x; \widehat{\Theta}_x)$ indicates that $\Theta_x$ and $\widehat{\Theta}_x$ do agree on the zero entries and for the nonzero entries they have the same sign.   
Theorems \ref{thm:binary main theorem} and \ref{thm:uncoded} state that the event $\mathcal{M}(\Theta_x, \widehat{\Theta}_x)$ occurs with high probability. Before stating the theorems we should define some properties of the underlying GGM. We denote the maximum degree of the underlying graph structure by $\Delta$. The minimum absolute value of the edges weighs in the precision matrix is denoted by $\theta_{\mathrm{min}}$ which is
\begin{equation}
    \theta_{\mathrm{min}} = \min_{(i,j) \in E(\Theta_x)} \vert\left(\Theta_x\right)_{ij}\vert.
\end{equation}


\begin{theorem} \label{thm:binary main theorem}
Consider a normal distribution satisfying the incoherence Assumption \ref{assump: incoherence} and \ref{assump:covariance control} with parameter $\alpha \in (0, 1]$. Let $\widehat{\Theta}_x$ be the  solution of the log-determinant program \eqref{eq:ML problem} with sample covariance $\widehat{S}_x$ in \eqref{eq:sample cov matrix binary} and regularization parameter $\lambda_n = (8\pi/\alpha) \sqrt{\frac{1}{2n}\ln\frac{2}{\epsilon}}$ for some $0 < \epsilon \leq d^{-2}$. Then,
\begin{enumerate}[label=(\alph*)]
\item
If the sample size is lower bounded as
\begin{equation}
n > C_{\mathrm{sign}}^2~\Delta^2~\left(1+\frac{8}{\alpha}\right)^2 \ln\frac{2}{\epsilon},
\end{equation}
where
\begin{equation*}
C_{\mathrm{sign}} = 3\sqrt{2} \pi \max\{\kappa_{\Sigma}\kappa_{\Gamma},\kappa_{\Sigma}^3\kappa_{\Gamma}^2\},   
\end{equation*}
then with probability at least $1-d^2 \epsilon$, the edge set specified by $\widehat{\Theta}_x$ is a subset of the true edge set.

\item
If the sample size satisfies the lower bound
 \begin{equation}
        n > T^2_{\mathrm{sign}} \left(1+\frac{8}{\alpha}\right)^2 \ln\frac{2}{\epsilon},
    \end{equation}
    where 
    \begin{equation*}
        T_{\mathrm{sign}} = \sqrt{2}\pi \max\{\kappa_{\Gamma}\theta_{\mathrm{min}}^{-1}, 3 \Delta~\max\{\kappa_{\Sigma}\kappa_{\Gamma},\kappa_{\Sigma}^3\kappa_{\Gamma}^2\}\},
    \end{equation*}
	then,
	\begin{equation}
	\Pr\left(\mathcal{M}(\widehat{\Theta}_x; \Theta_x)\right) \geq 1-d^2\epsilon.
	\end{equation} 
\end{enumerate}
\end{theorem}

\begin{remark}
	Note that the proposed sign method is applicable for any channel with capacity greater than or equal to 1 bit.
\end{remark}

\section{Uncoded Method} \label{sec:uncoded method}
In this section, we assume that each local machine puts its local data into the channel without any source or channel coding. The central machine estimates the underlying graph structure using received data from the channel. At the central machine, no source or channel decoding is used. It infers the structure directly from the output samples of the channel.

As described in Section \ref{sec:channel model}, we consider a SIMO-GMAC. Each local machine can transmit two consequent samples by each channel use. More precisely, denoting the channel input symbol by $\vect s = \vect s_R + j \vect s_I$, each local machine can put two consequent samples as real and imaginary parts of the input symbol. Therefore, at the central machine, $n/2$ vectors are received that each one is $2d$-dimensional.

In this way, the  equation \eqref{eq:y=Hx+z} can be decomposed as
\begin{align}
\vect y = (H_R + j H_I) (\vect s_R + j \vect s_I) + (\vect z_R + j \vect z_I),
\end{align}
where $j^2 = -1$. Hence, it can be rewritten in a block-matrix form as
\begin{equation} \label{eq:y=Hx+z matrix}
\begin{bmatrix}
\vect y_R \\
\vect y_I
\end{bmatrix}
= 
\begin{bmatrix}
H_R & -H_I \\
H_I & H_R
\end{bmatrix}
\begin{bmatrix}
\vect s_R \\
\vect s_I
\end{bmatrix}
+
\begin{bmatrix}
\vect z_R \\
\vect z_I
\end{bmatrix},
\end{equation}
where $H_R, H_I$ are $d \times d$ real matrices, and all the real and imaginary part vectors are $d$-dimensional. In this way, two source samples can be transmitted per channel use: a sample is put into the real part and the other is put into the imaginary part. In particular, $\vect s_R + j\vect s_I = \sqrt{\frac{p}{2}}(\vect x_R+j\vect x_I)$, where $\vect x_R$ and $\vect x_I$ are two independent samples from the source. In this way, the transmit power constraints are satisfied.  

The central machine estimates the conditional dependencies of the vectors $\vect x$ using the received vectors $\vect y$.

\subsection{Approximating the Sample Covariance}
Defining matrix $\widetilde{H} = \sqrt{\dfrac{p}{2}} 
\begin{bmatrix}
H_R & -H_I \\
H_I & H_R
\end{bmatrix}$, $\tilde{\vect x} = \begin{bmatrix}
\vect x_R \\
\vect x_I
\end{bmatrix}$, and $\tilde{\vect y} = \begin{bmatrix}
\vect y_R \\
\vect y_I
\end{bmatrix}$. When transmitting samples through the channel, if we put two independent samples as real and imaginary parts of the vector $\tilde{\vect x}$, then the covariance matrix of $\tilde{\vect x}$ is
\begin{equation}
Q_{\tilde x} = 
\begin{bmatrix}
Q_x & 0 \\
0	& Q_x
\end{bmatrix}.
\end{equation}
On the other hand, according to equation \eqref{eq:y=Hx+z matrix}, we have
\begin{equation}
Q_{\tilde x} = \widetilde{H}^{-1} Q_{\tilde y} \widetilde{H}^{-T} - \sigma_z^2 \widetilde{H}^{-1} \widetilde{H}^{-T},
\end{equation}
where $Q_{\tilde x}$ and $Q_{\tilde y}$ is the covariance matrix of $\tilde{\vect x}$ and $\tilde{\vect y}$, respectively. By substituting the sample covariance matrix of $\tilde{\vect y}$ into the above expression, we obtain an approximation for the sample covariance matrix of $\tilde{\vect x}$, as
\begin{equation} \label{eq:S_x tilde}
S_{\tilde x} = \widetilde{H}^{-1} S_{\tilde y} \widetilde{H}^{-T} - \sigma_z^2 \widetilde{H}^{-1} \widetilde{H}^{-T}.
\end{equation}
\begin{lemma} \label{lemma:sample cov err tilde}
Given $n$ i.i.d. samples of the vector $\tilde{\vect y}$. Then, for the sample covariance matrix $S_{\tilde x}$ in \eqref{eq:S_x tilde} we have
\begin{equation}
\Pr\left(\vert (S_{\tilde x})_{jk} - (Q_{\tilde x})_{jk} \vert \geq \delta \right) \leq 4\exp\left(\frac{-n\delta^2}{c}\right),
\end{equation}
where
\begin{equation}
    c = 3200 \left(1 + \sigma_z^2 / \lambda^2_{\mathrm{min}}(\widetilde{H})\right)^2,
\end{equation}
and $\lambda_{\mathrm{min}}(\widetilde{H})$ is the minimum eigenvalue of $\widetilde{H}$.
\end{lemma}
\begin{proof}
	We define the random variable $\vect w$ as follows
	\begin{equation*}
	\vect w = \widetilde H^{-1} \vect y = \vect x + \widetilde{H}^{-1} \vect z.
	\end{equation*}
	It is clear that $\vect w \sim \mathcal{N}(\vect 0, Q_w)$, where $Q_w=Q_x + \sigma_z^2 \widetilde{H}^{-1}\widetilde{H}^{-T}$. Denoting $\vect w = [w_1, \cdots, w_p]^T$, $w_j/\sqrt{(Q_w)_{jj}}$ is a standard normal variable which is sub-Gaussian with parameter $1$. Thus, according to Lemma 1 in \cite{ravikumar2011high}, we have
	\begin{align*}
	\Pr\left(\vert (S_w)_{jk} - (Q_w)_{jk} \vert \geq \delta \right) \leq \exp\left\{-\frac{n\delta^2}{3200 \max_i (Q_w)_{ii}^2}\right\},
	\end{align*}
	where $S_w = \widetilde{H}^{-1}S_y \widetilde{H}^{-T}$ is the sample covariance over $\vect w^{(i)} = \widetilde{H}^{-1} \vect y^{(i)}$ samples. On the other hand, we have 
	\begin{align}
	\Pr\left(\vert \left(S_{\tilde x} - Q_{\tilde x}\right)_{jk} \vert \geq \delta \right) &= \Pr\left(\left\vert \left( \widetilde{H}^{-1}\left(S_{\tilde y} - Q_{\tilde y}\right)\widetilde{H}^{-T}\right)_{jk} \right\vert \geq \delta \right) \nonumber \\
	&=\Pr\left(\vert \left(S_w - Q_w\right)_{jk} \vert \geq \delta \right) \nonumber \\
	&\leq\exp\left\{-\frac{n\delta^2}{3200 \max_i (Q_w)_{ii}^2}\right\}  \label{eq:S_w bound}
	\end{align}
	By finding an upper bound on  $\max_i (Q_w)_{ii}^2$, we can obtain an upper bound on the above probability. 
	\begin{align}
	\max_i{(Q_{w})_{ii}} &= \max_i{(Q_x + \sigma_{z}^2\widetilde{H}^{-1}\widetilde{H}^{-T})_{ii}} \nonumber \\
	&\leq \max_i{(\left|Q_{x}\right|+\left|\sigma_{z}^2\widetilde{H}^{-1}\widetilde{H}^{-T}\right|)_{ii}} \nonumber \\
	&\leq \max_i{(Q_{x})_{ii}}+\max_j{(\sigma_{z}^2\widetilde{H}^{-1}\widetilde{H}^{-T})_{jj}}. \label{eq:max(Q_ii) upper bound}
	\end{align}
	By decomposing $\widetilde{H}^{-1}\widetilde{H}^{-T}$ as $U\Lambda^2 U^T$, we have
	\begin{align*}
	\max_i{(\sigma_{z}^2\widetilde{H}^{-1}\widetilde{H}^{-T})_{ii}} &= \max_i (\sigma_{z}^2 U\Lambda^2 U^T)_{ii} \\
	&\leq \max_i (\sigma_z^2 \lambda_{\mathrm{max}}^2(\widetilde{H}^{-1}) U U^T)_{ii} \\
	&=  \frac{\sigma_{z}^2}{\lambda_{\mathrm{min}}^2(\widetilde{H})}.
	\end{align*}
	By combining the above bound and \eqref{eq:max(Q_ii) upper bound}, and substituting into \eqref{eq:S_w bound} we obtain the claimed bound in the lemma.
\end{proof}

Next, we define the approximate sample covariance matrix of $\vect x$ as
\begin{equation} \label{eq:S_x joint}
\hat S_x = \frac{1}{2} \left(\begin{bmatrix}
I_d & 0_d
\end{bmatrix}
S_{\tilde x} 
\begin{bmatrix}
I_d \\
0_d
\end{bmatrix}
+
\begin{bmatrix}
0_d & I_d
\end{bmatrix}
S_{\tilde x}
\begin{bmatrix}
0_d \\
I_d
\end{bmatrix}\right),
\end{equation}
where $0_d$ is a $d\times d$ zero matrix.
\begin{lemma}
	For $\hat S_x$ defined in \eqref{eq:S_x joint}, we have
	\begin{equation}
		\Pr\left(\vert (S_x)_{jk} - (Q_x)_{jk} \vert \geq \delta \right) \leq 8\exp\left(\frac{-n\delta^2}{2c}\right).
	\end{equation} 
\end{lemma}
\begin{proof}
From \eqref{eq:S_x joint}, it is clear that $$(\hat S_{x})_{jk} = \frac{1}{2}\left((S_{\tilde x})_{jk} + (S_{\tilde x})_{(d+j)(d+k)}\right).$$ Thus, we have
\begin{align*}
&\Pr\left(\vert (\hat S_x)_{jk} - (Q_x)_{jk} \vert \geq \delta \right) \\
&\quad= \Pr\left(\vert (S_{\tilde x})_{jk} + (S_{\tilde x})_{(d+j)(d+k)} - 2(Q_{\tilde x})_{jk} \vert \geq 2\delta \right) \\
&\quad\leq \Pr\left(\vert (S_{\tilde x})_{jk} - (Q_{\tilde x})_{jk}\vert + \vert (S_{\tilde x})_{(d+j)(d+k)} - (Q_{\tilde x})_{jk} \vert \geq 2\delta \right) \\
&\quad \leq \Pr\left(\vert (S_{\tilde x})_{jk} - (Q_{\tilde x})_{jk}\vert \geq \delta \right) + \\ 
&\quad\qquad \Pr\left(\vert (S_{\tilde x})_{(d+j)(d+k)} - (Q_{\tilde x})_{jk}\vert \geq \delta \right) \\
&\quad \leq 8\exp\bigg(\frac{-n\delta^2}{2c}\bigg),
\end{align*}
where the last inequality is obtained by the error bound of Lemma \ref{lemma:sample cov err tilde} for the sample size $n/2$.
\end{proof}
Note that the matrix $\widehat{S}_x$ in \eqref{eq:S_x joint} is not necessarily positive semi-definite. But this does not affect the convexity of the optimization problem \eqref{eq:ML problem}.

By substituting $\widehat{S}_x$ from \eqref{eq:S_x joint} into the \eqref{eq:ML problem}, we can solve the $\ell_1$ regularized maximum likelihood problem and obtain a sparse solution for the precision matrix $\Theta_x$. Similar to Theorem \ref{thm:binary main theorem}, we can guarantee that Uncoded method can correctly recover the underlying graph structure with high probability.

\begin{theorem} \label{thm:uncoded}
Consider a normal distribution satisfying the incoherence Assumption \ref{assump: incoherence} and \ref{assump:covariance control} with parameter $\alpha \in (0, 1]$. Let $\widehat{\Theta}_x$ be the  solution of the log-determinant program \eqref{eq:ML problem} with sample covariance $\widehat{S}_x$ in \eqref{eq:S_x joint} and regularization parameter $\lambda_n = (8\pi/\alpha) \sqrt{\frac{1}{2n}\ln\frac{2}{\epsilon}}$ for some $0 < \epsilon \leq d^{-2}$.
\begin{enumerate}[label=(\alph*)]
\item
If the sample size is lower bounded as
\begin{equation}
n > C_{\mathrm{uncoded}}^2~\Delta^2\left(1+\frac{8}{\alpha}\right)^2 \ln\frac{8}{\epsilon},
\end{equation}
where
\begin{equation*}
C_{\mathrm{uncoded}} = 6 \sqrt{2c}~ \max\{\kappa_{\Sigma}\kappa_{\Gamma},\kappa_{\Sigma}^3\kappa_{\Gamma}^2\}.    
\end{equation*}
then with probability at least $1-d^2 \epsilon$, the edge set specified by $\widehat{\Theta}_x$ is a subset of the true edge set.

\item
If the sample size satisfies the lower bound
 \begin{equation}
        n > T^2_{\mathrm{uncoded}} \left(1+\frac{8}{\alpha}\right)^2 \ln\frac{8}{\epsilon},
    \end{equation}
    where 
    \begin{equation*}
        T_{\mathrm{uncoded}} = 2\sqrt{2c} \max\{\kappa_{\Gamma}\theta_{\mathrm{min}}^{-1}, 3 \Delta~\max\{\kappa_{\Sigma}\kappa_{\Gamma},\kappa_{\Sigma}^3\kappa_{\Gamma}^2\}\}
    \end{equation*}
	then,
	\begin{equation}
	\Pr\left(\mathcal{M}(\widehat{\Theta}_x; \Theta_x)\right) \geq 1-d^2\epsilon.
	\end{equation} 
\end{enumerate}
\end{theorem}

\begin{remark}
Since the  channel has real and imaginary parts, 2 samples can be transmitted by each channel access (i.e. each local machine can transmit $n$ samples by $n/2$ channel uses). Therefore, if the sample generation rate at the source is less than or equal to twice of the channel's rate, the machines can transmit all samples without any sample loss.
\end{remark}


\section{Experiments} \label{sec:experiment}
In this section, the performance of our proposed methods is evaluated by performing several experiments. In our simulations\footnote{The source code is available at \url{https://github.com/ArminKaramzade/distributed-sparse-GGM}.}, the \emph{glasso} package \cite{friedman2008sparse} is used to solve $\ell_1$-regularized MLE of the precision matrix. This package is based on the block coordinate descent algorithm proposed by \cite{banerjee2008model}.

In our experiments, a sparse random precision matrix is generated as follows. First, we generate a random sparse graph with a fixed probability of the edge presence, say $0.1$, and also set its maximum node degree to $\Delta = 5$. Then, we choose edge weights uniformly in $[-1, 1]$ for the symmetric precision matrix $\Theta$. Next, we make it positive definite matrix by adding a scaled identity matrix. Finally, we normalize the precision matrix to set the variances to $1$. Also, we ensure that the generated matrix satisfies Assumption \ref{assump: incoherence}.

We employ the True Positive and False Positive Rates (TPR and FPR, respectively) as our performance measures. TPR is defined as the percentage of the predicted edges (non-zero off-diagonal entries in the precision matrix) that correctly detected. Similarly, FPR is the percentage of the predicted non-edges (zero entries in the precision matrix) that incorrectly detected.

We have experimentally observed
that $\lambda_n^{\text{*signs}}\approx 4\lambda_n^{\text{*original}}$ and $\lambda_n^{\text{*uncoded}}\approx \frac{2}{3}\lambda_n^{\text{*original}}$, where $\lambda_n^{\text{*signs}}$, $\lambda_n^{\text{*uncoded}}$, and $\lambda_n^{\text{*original}}$ are the best regularization parameter for the signs, uncoded, and original data, respectively. 

In order to have a fair comparison between Signs and Uncoded methods, we have used identical parameters for the channel. More precisely, we assume $H=I$ which ensures all the local machines have identical bit rates. We set the bit rate of each local machine, i.e. $R_j$, to 2 bits. Thus, in the proposed methods, each local machine can transmit $n$ bits by $n/2$ channel uses. According to \eqref{eq:channel rates region}, in order to achieve the bit-rate of 2, the signal to noise ratio (SNR) should set to 3 (i.e. $\frac{p}{\sigma^2_z}=3$).

In  the first experiment, we evaluate the performance of our methods with respect to the dimension $d$. \figurename{}~\ref{fig:error-in-dim} shows TPR and FPR as a function of $d$ for sample sizes $n=1000$ and $10,000$. In this experiment, the error curves are averaged over 20 different random graphs and for each graph the TPR and FPR are averaged over 10 different random samples. As can be seen from \figurename{}~\ref{fig:error-in-dim}, Signs method outperforms Uncoded method. However, all three methods have approximately the same FPR. 

\figurename{}~\ref{fig:error-in-sam} reflects the performance of the methods as a function of the sample size $n$ for $d=50$ and $d=100$. As can be seen, by increasing the sample size $n$, the performance of all methods increases. In this experiment again Signs method outperforms the uncoded scheme.

\begin{figure*}[t]
  \centering
  \subfloat[$n=1000$]{\includegraphics[width=.24\linewidth]{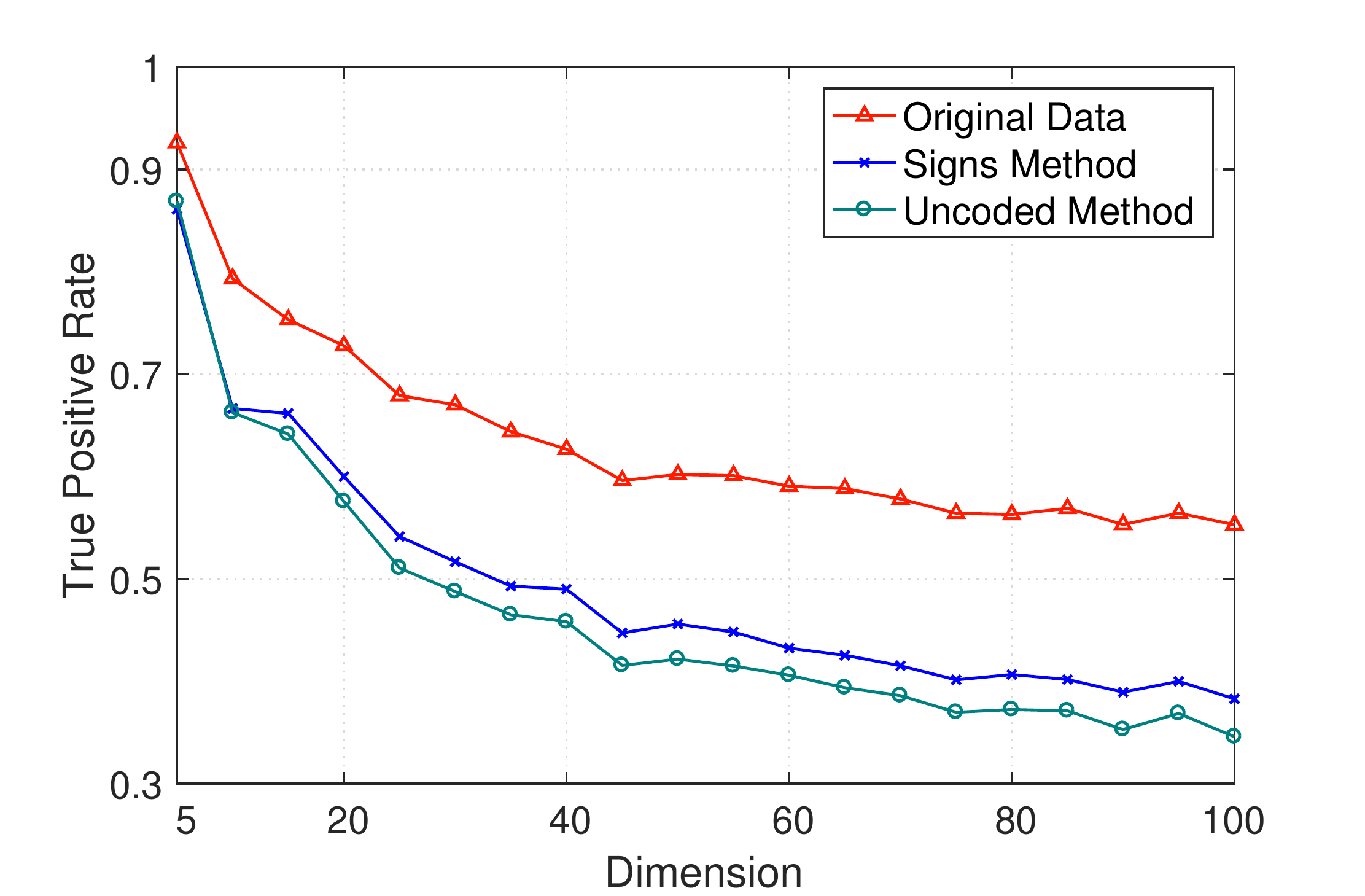}  \label{fig:TPR-in-dim-1000}}
  \subfloat[$n=1000$]{\includegraphics[width=.24\linewidth]{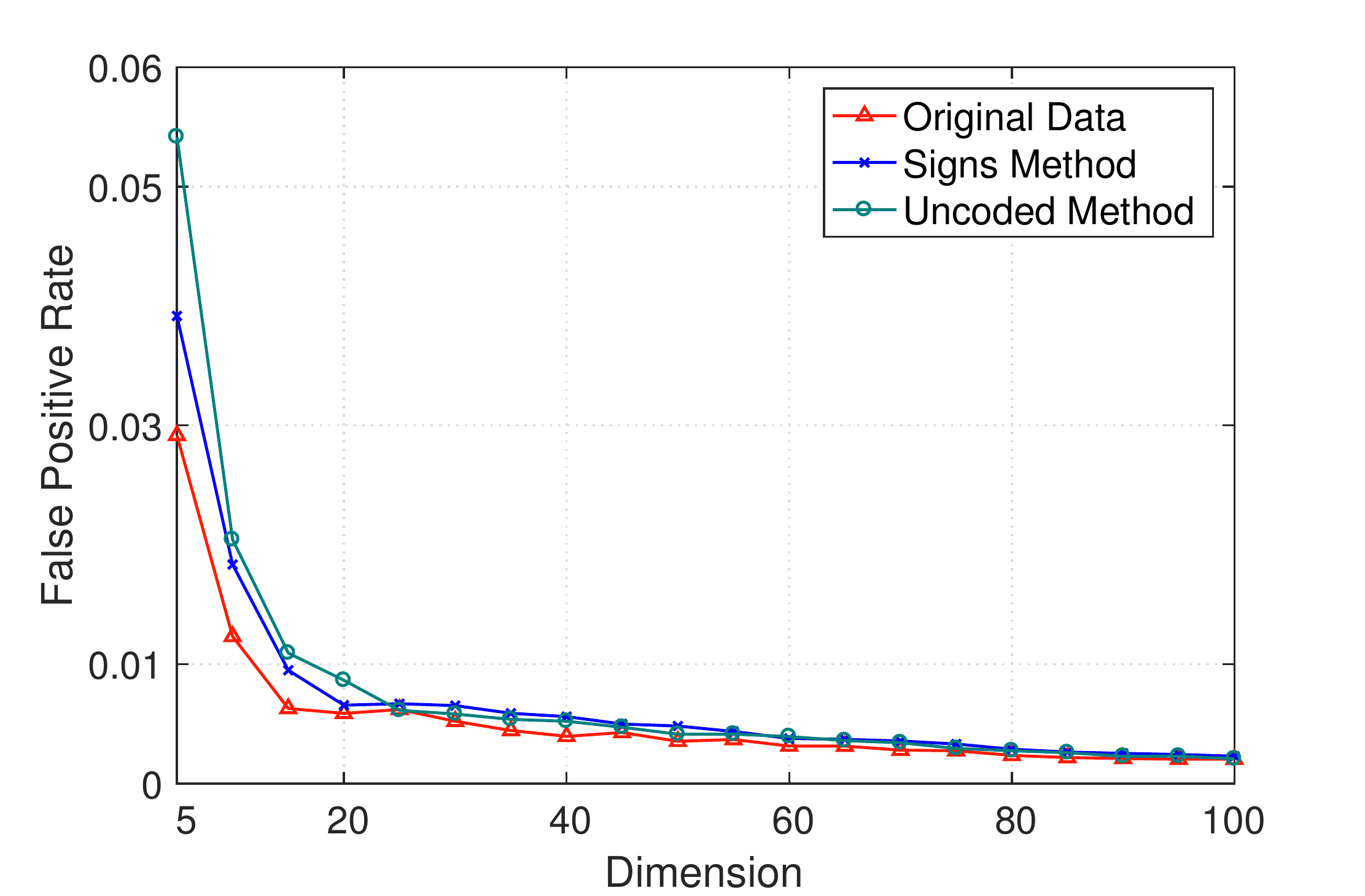}  \label{fig:FPR-in-dim-1000}}
  \subfloat[$n=10,000$]{\includegraphics[width=.24\linewidth]{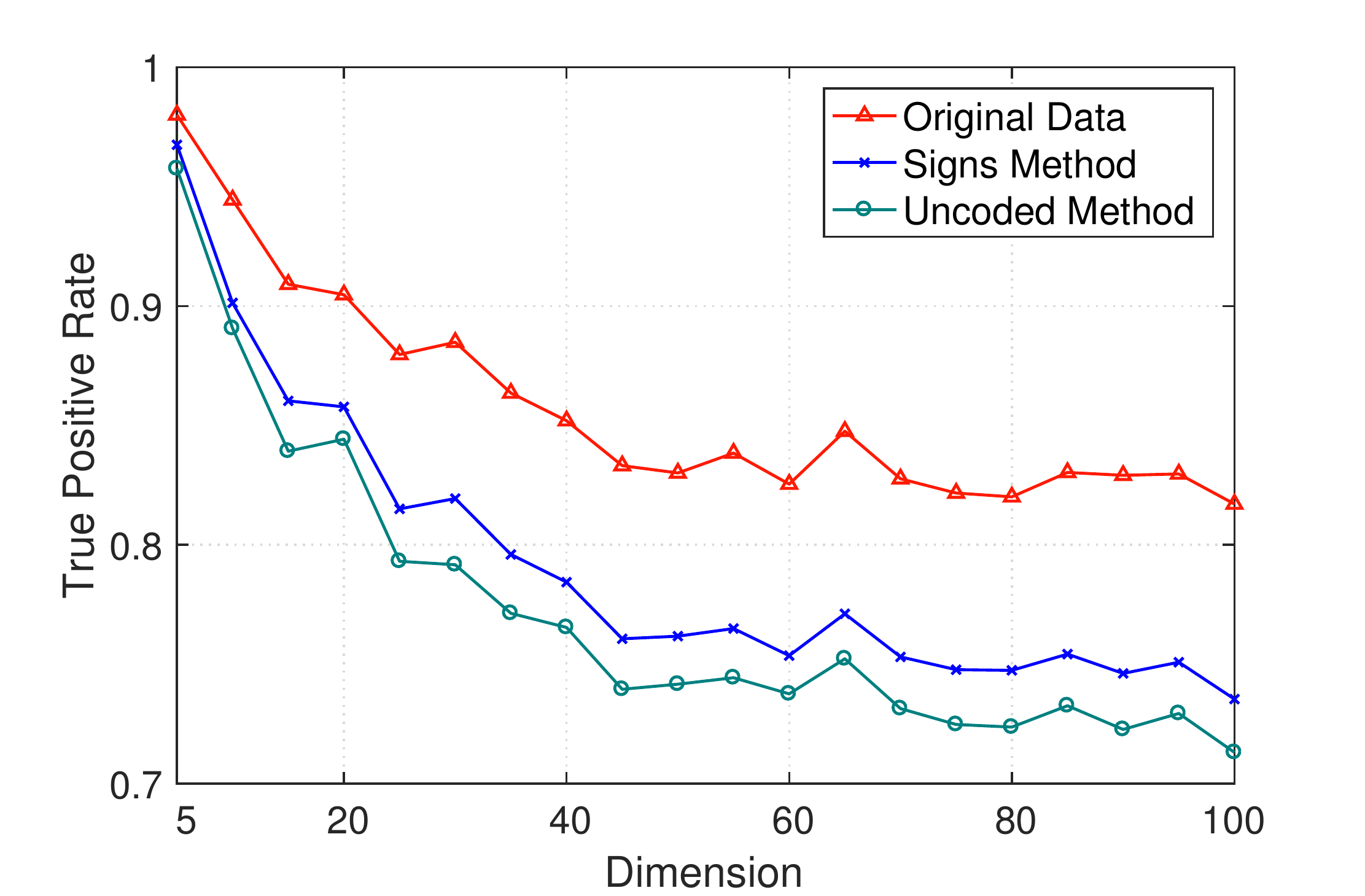}  \label{fig:TPR-in-dim-10000}}
  \subfloat[$n=10,000$]{\includegraphics[width=.24\linewidth]{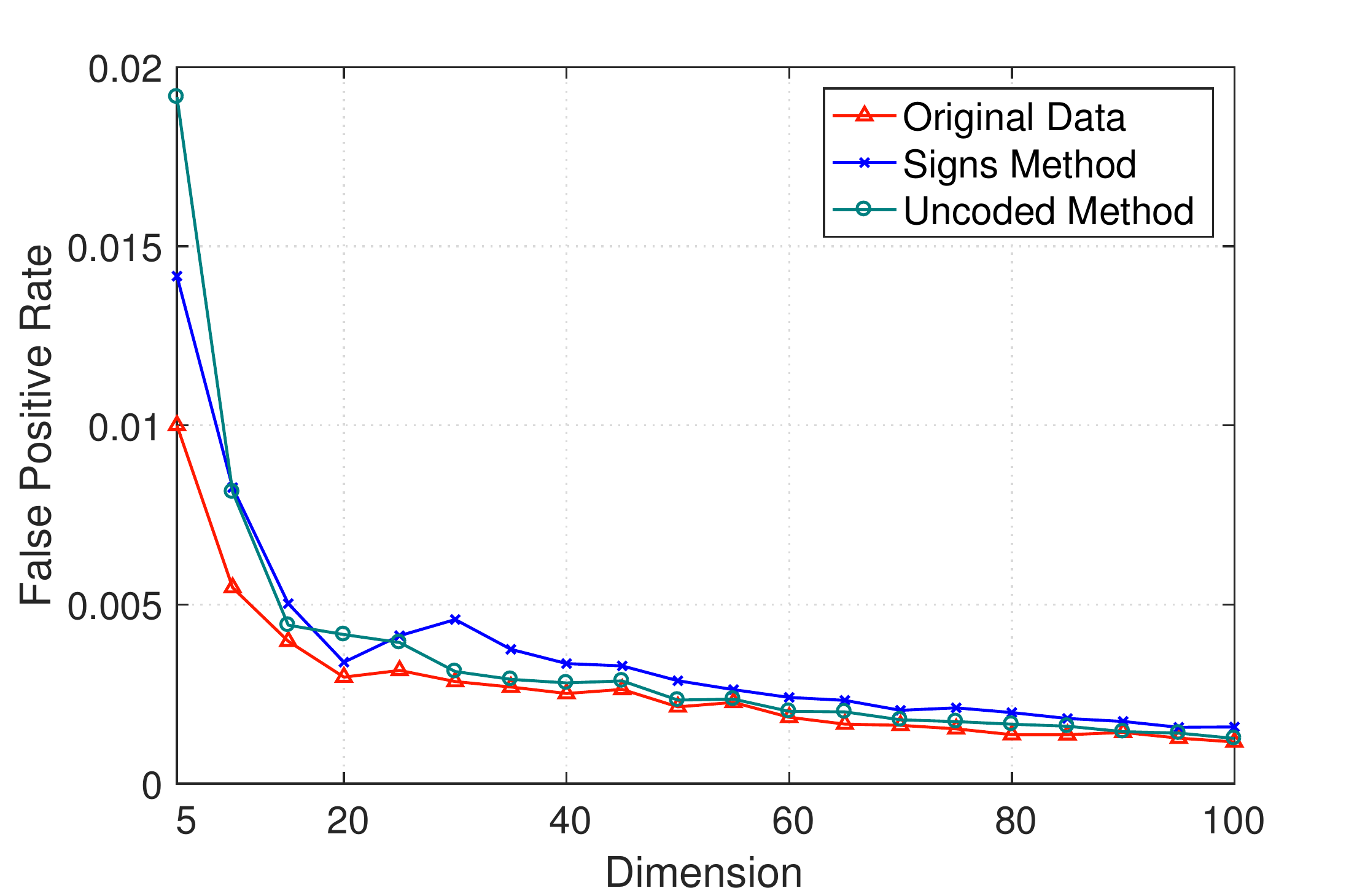}  \label{fig:FPR-in-dim-10000}}
 \caption{TPR and FPR as a function of dimension $d$ for sample sizes $n=1000$ and $n=10,000$}
  \label{fig:error-in-dim}
\end{figure*}

 \begin{figure*}[t]
  \centering
  \subfloat[$d=50$]{\includegraphics[width=.24\linewidth]{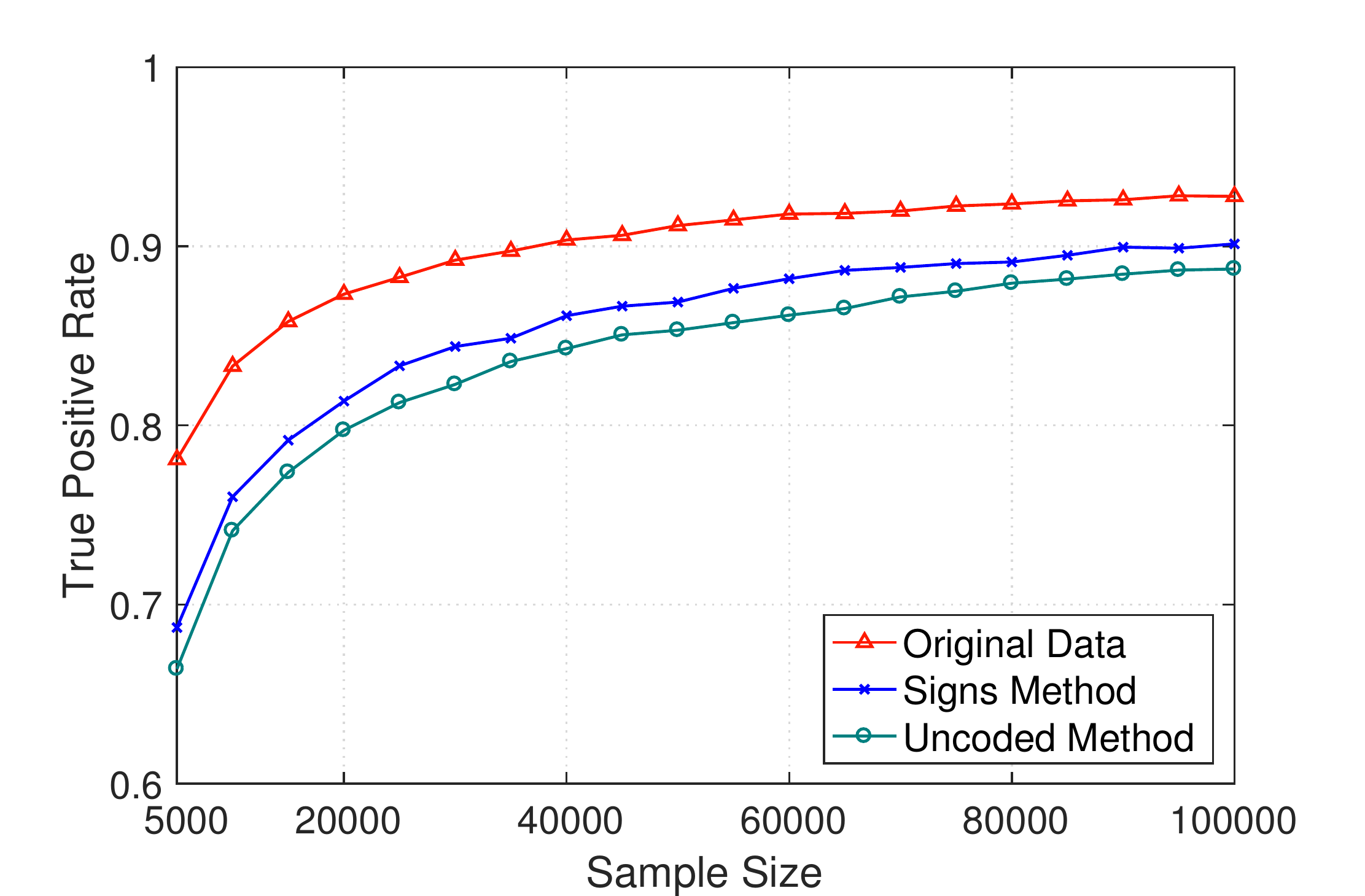}  \label{fig:TPR-in-sam-50}}
  \subfloat[$d=50$]{\includegraphics[width=.24\linewidth]{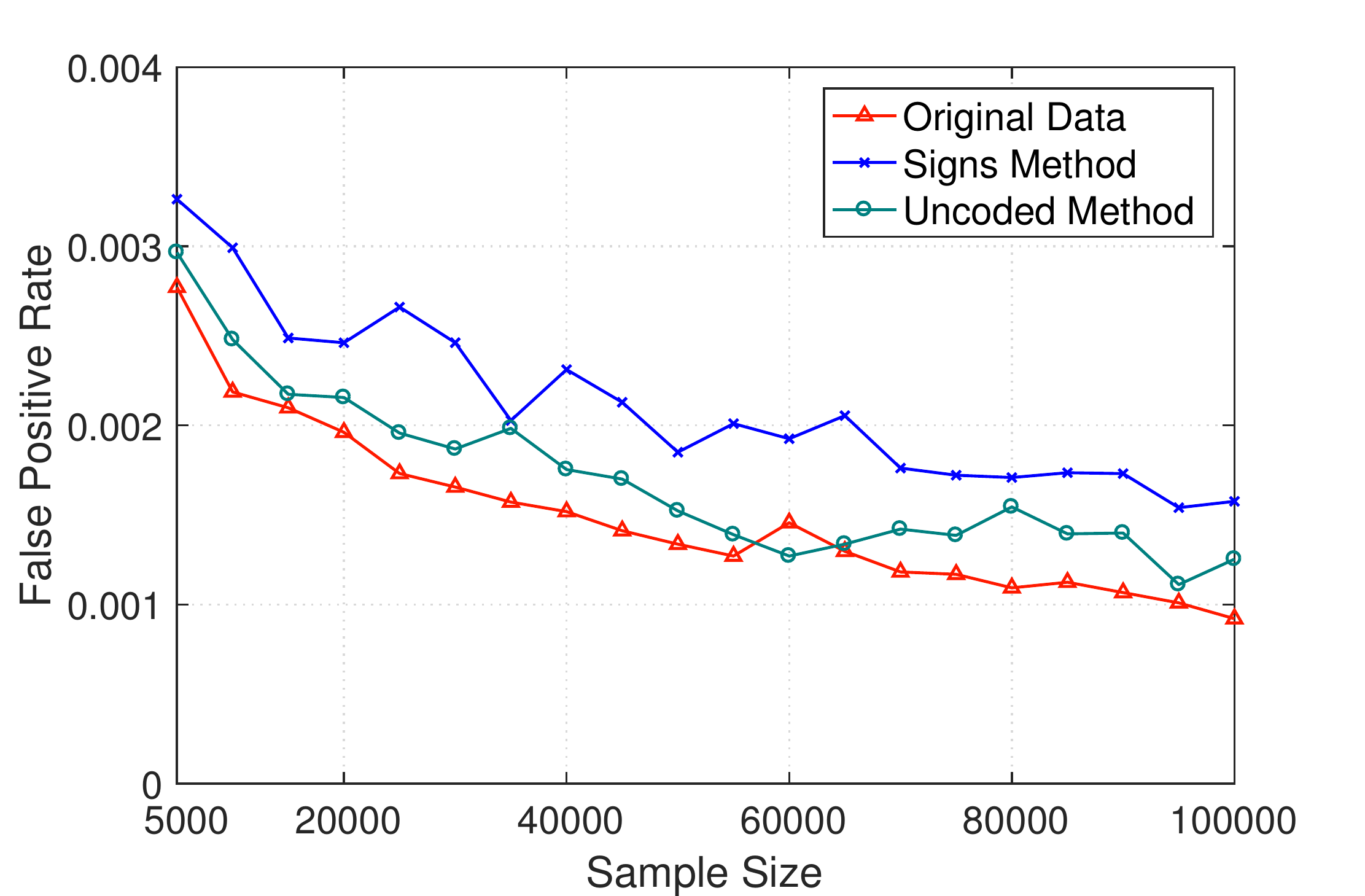}  \label{fig:FPR-in-sam-50}}
  \subfloat[$d=100$]{\includegraphics[width=.24\linewidth]{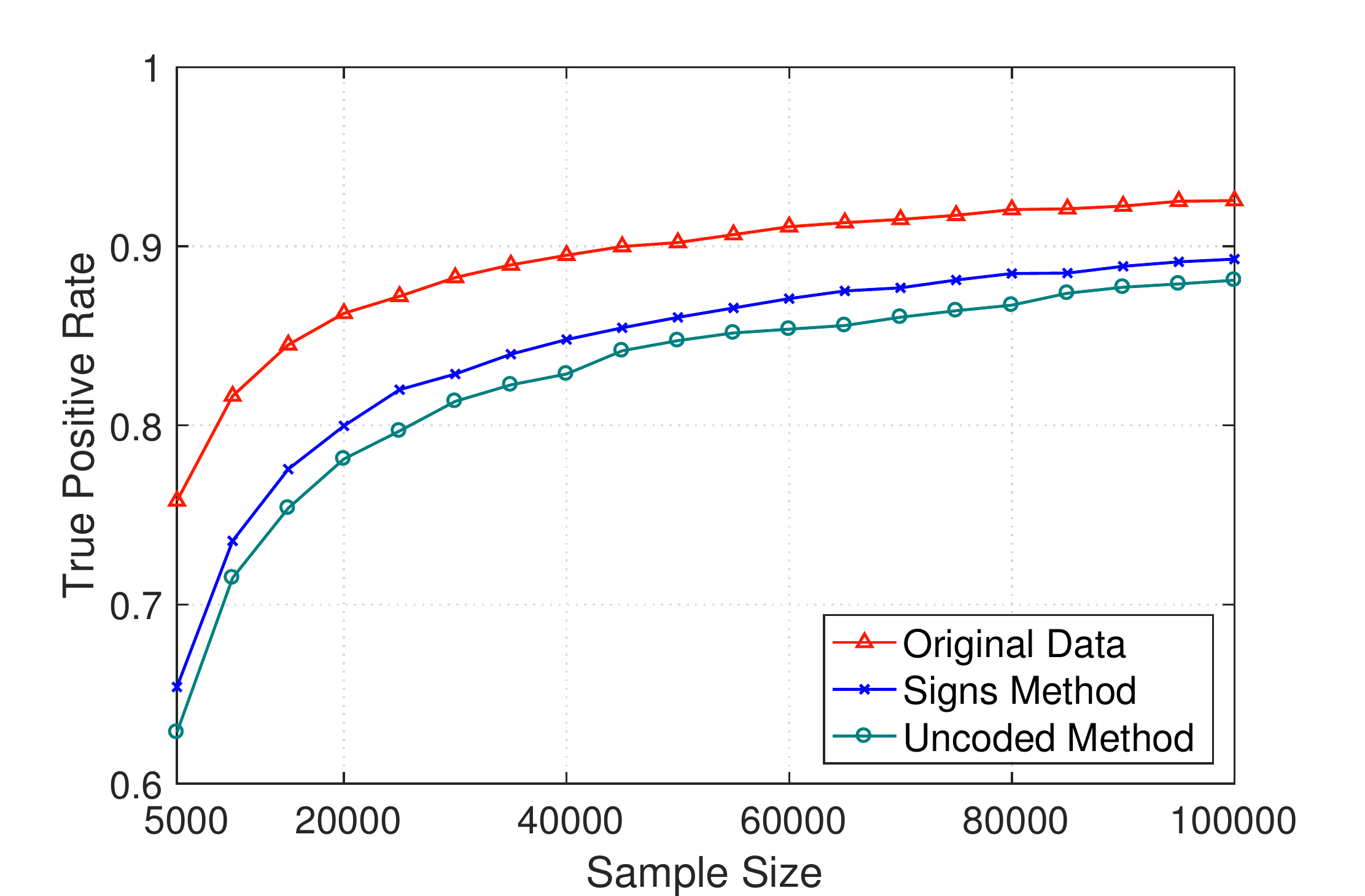}  \label{fig:TPR-in-sam-100}}
  \subfloat[$d=100$]{\includegraphics[width=.24\linewidth]{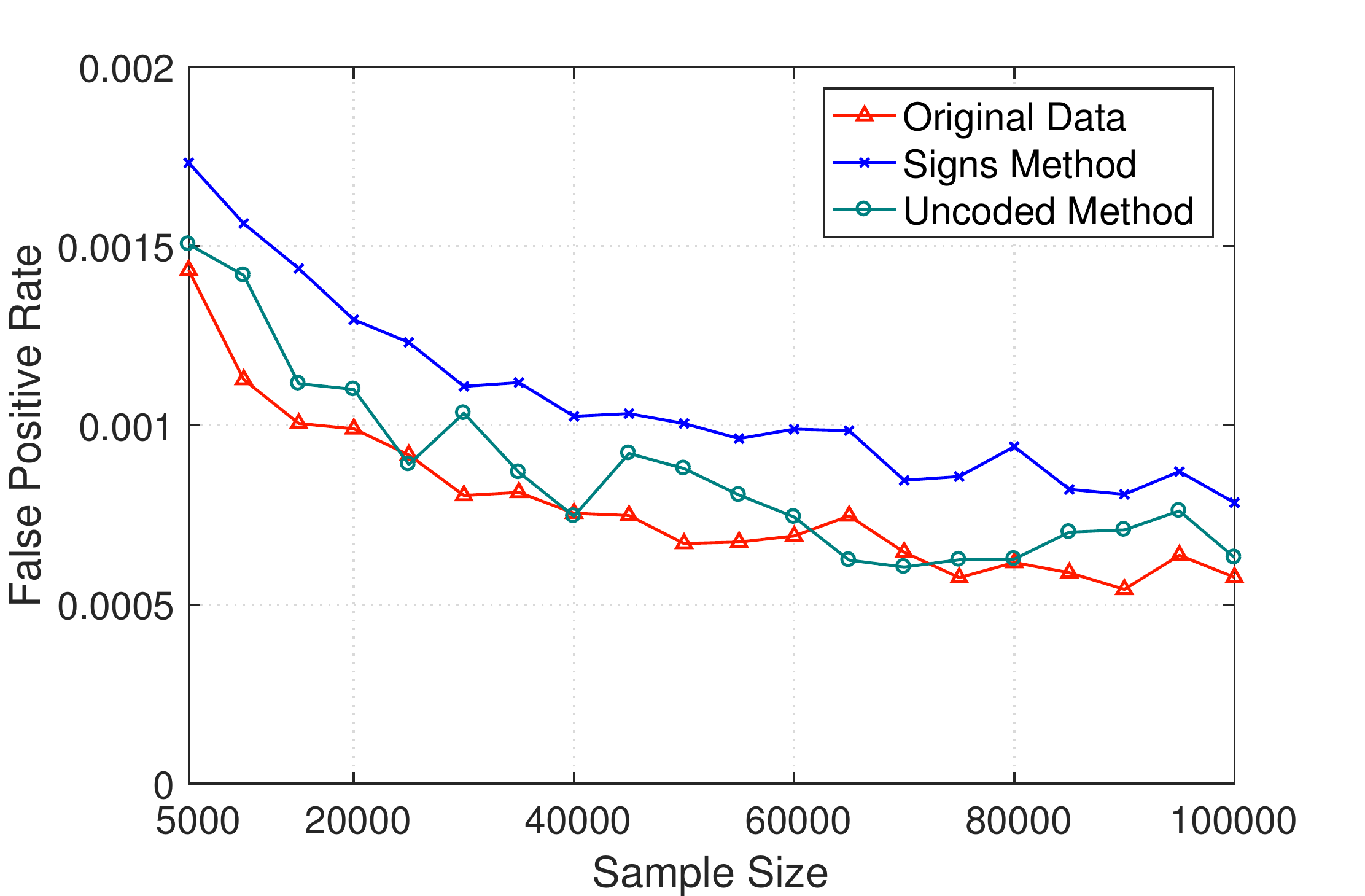}  \label{fig:FPR-in-sam-100}}
  \caption{TPR and FPR as a function of the sample size $n$ for random graphs with 50 and 100 nodes.}
  \label{fig:error-in-sam}
\end{figure*}

In \figurename{}~\ref{fig:pofe}, the probability of perfect structure recovery for a star-shaped graph is depicted. In this experiment, the underlying star graph consists of $d=70$ nodes. The precision matrix is generated as the inverse of a covariance matrix with $(Q_x)_{ij}=\frac{1}{4}$ for all $(i, j) \in \mathcal{E}$ which satisfies Assumption \ref{assump: incoherence}. The probability of perfect recovery is estimated by running the proposed methods 100 times and counting the number of times that the structure is recovered exactly. As can be seen from the figure, all methods recover the structure exactly for large enough sample sizes as claimed by theorems \ref{thm:binary main theorem} and \ref{thm:uncoded}. 
		
\begin{figure}[t]
    \centering
    \includegraphics[width=.8\linewidth]{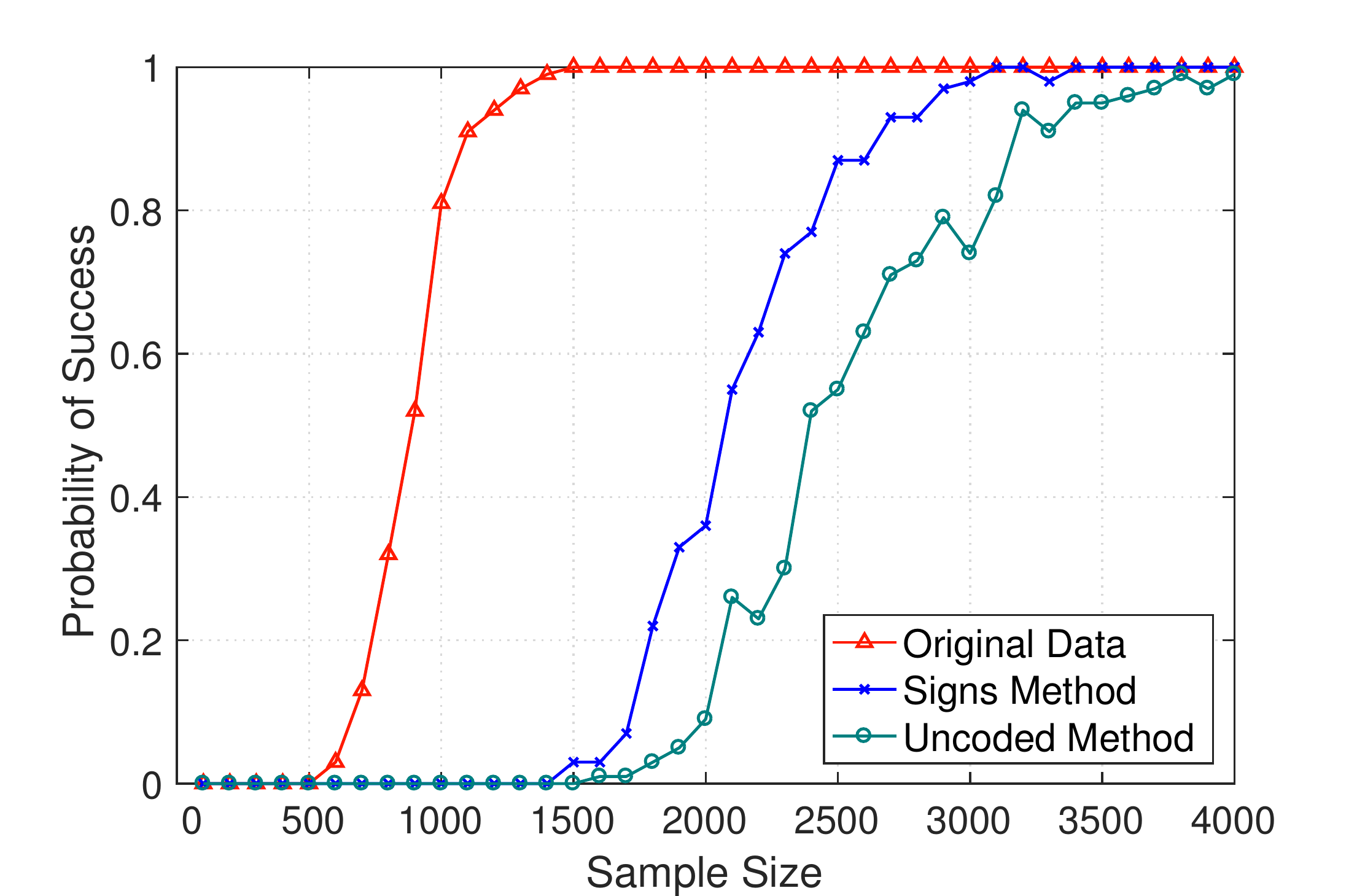}
    \caption{Probability of perfect structure recovery for the star graph with $d=70$.}
    \label{fig:pofe}
\end{figure}
	
In \figurename{}~\ref{fig:TPR-in-SNR}, we measure the TPR of Uncoded method for different values of the SNR. The experiment is performed on a random graph with $d=40$ nodes with maximum degree of $\Delta=5$ and $n=10,000$. In this experiment, we have generated the channel matrix $H$ with entries drawn from i.i.d. standard normal samples. The TPR curve is averaged over 100 different channel matrices. As can be seen from \figurename{}~\ref{fig:TPR-in-SNR}, for SNR greater than 5, the performance of Uncoded method is very close to the TPR of the original data.

The FPR curve is not plotted, since the error values were negligible even for small SNR.

\begin{figure}[t]
\centering
  \includegraphics[width=.8\linewidth]{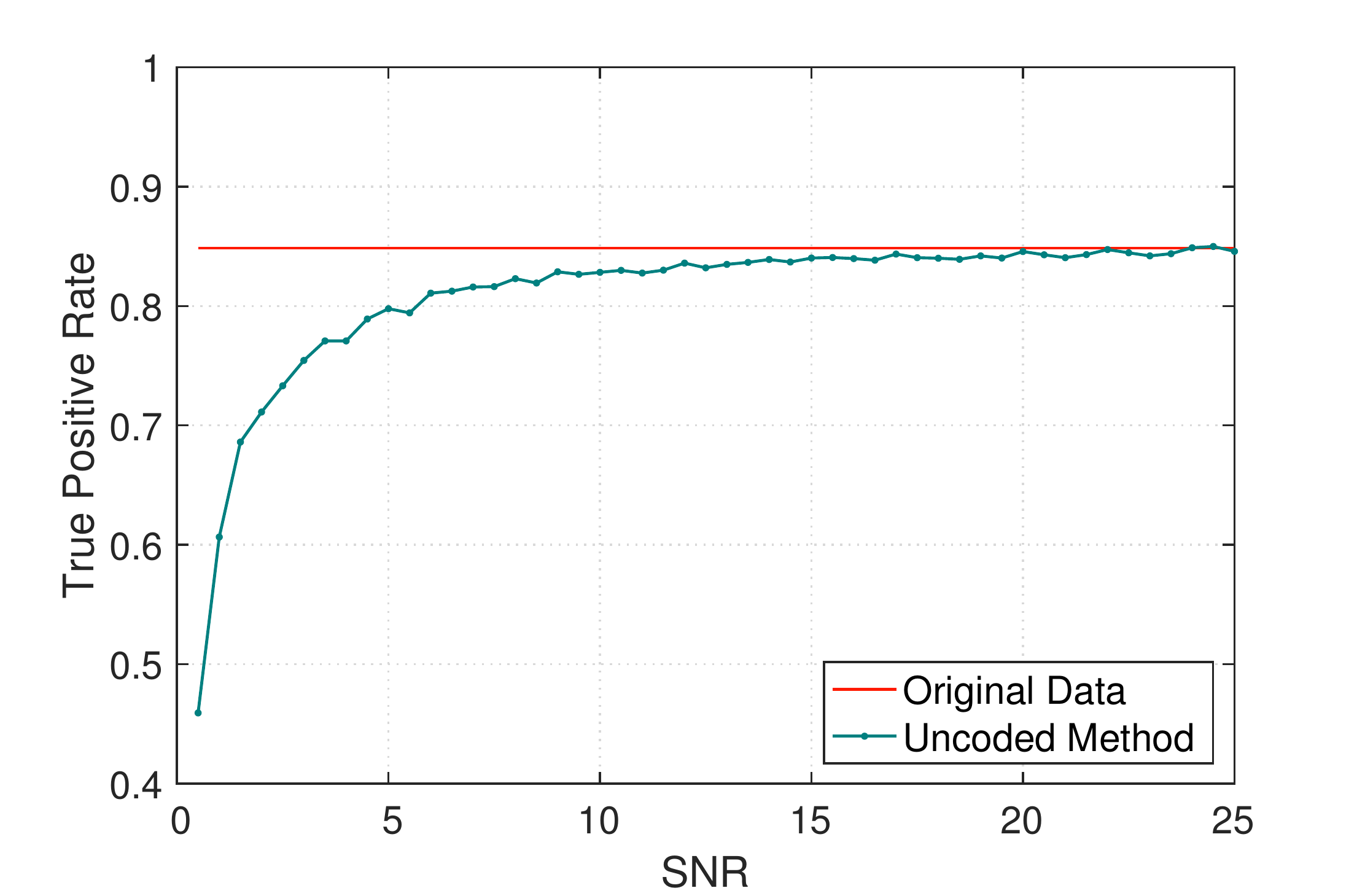}
  \caption{TPR as a function of SNR in Uncoded method for a random graph with $d=40$ and maximum degree $\Delta=5$.}
      \label{fig:TPR-in-SNR}
\end{figure}

\section{Conclusion} \label{sec:conclusion}
In this paper, we have studied the sparse structure learning of GGMs where the data are distributed across multiple local machines. Two methods are proposed to send information from the local machines to the central machine, namely, Signs and Uncoded methods. We have analytically and experimentally shown that the central machine can recover the underlying graph if large enough sample sizes are transmitted to the central machine. 

Our experiments show that, under the same conditions, Signs method outperforms the uncoded scheme. Both methods have small FPR which is close to the FPR obtained by the original data.

\section{Acknowledgment}
We thank Amir Najafi and Amir-Hossein Saberi for their valuable comments that greatly improved the manuscript.

\ifCLASSOPTIONcaptionsoff
  \newpage
\fi

\bibliographystyle{IEEEtran}
\bibliography{IEEEabrv,./refs}

\begin{IEEEbiography}[{\includegraphics[width=1in,height=1.25in,clip,keepaspectratio]{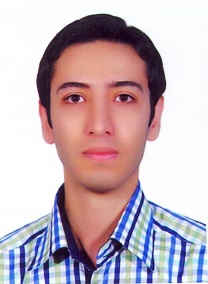}}]{Mostafa Tavassolipour}
	received the B.Sc. degree from Shahed University, Tehran, Iran, in 2009, and the M.Sc. degree from Computer Engineering department of Sharif University of Technology (SUT), Tehran, Iran, in 2011. Currently, He is a Ph.D. student of Artificial Intelligence program at Computer Engineering Department of Sharif University of Technology. His research interests include machine learning, image processing, information theory, content based video analysis, and bioinformatics.
\end{IEEEbiography}
\begin{IEEEbiography}[{\includegraphics[width=1in,height=1.25in,clip,keepaspectratio]{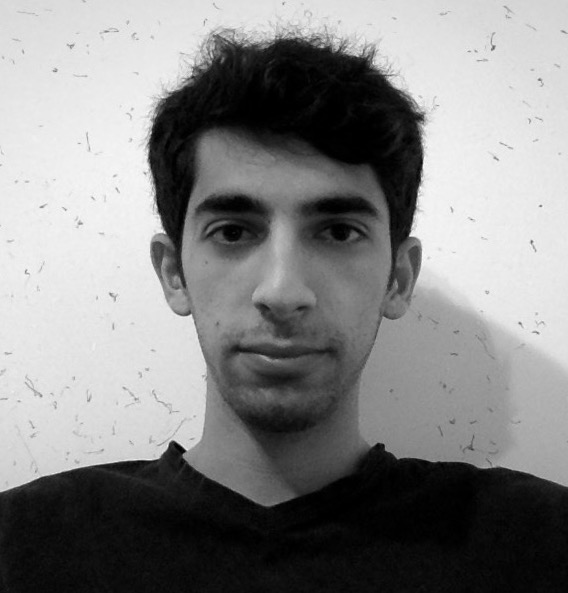}}]{Armin Karamzade} is M. Sc. student of Artificial Intelligence at the Sharif University of Technology, Tehran, Iran. He received his B. Sc. degree in Computer Engineering from Iran University of Science and Technology, Tehran, Iran, in 2017. His research interests include mostly machine learning, optimization, and statistic.
	
\end{IEEEbiography}
\begin{IEEEbiography}[{\includegraphics[width=1in,height=1.25in,clip,keepaspectratio]{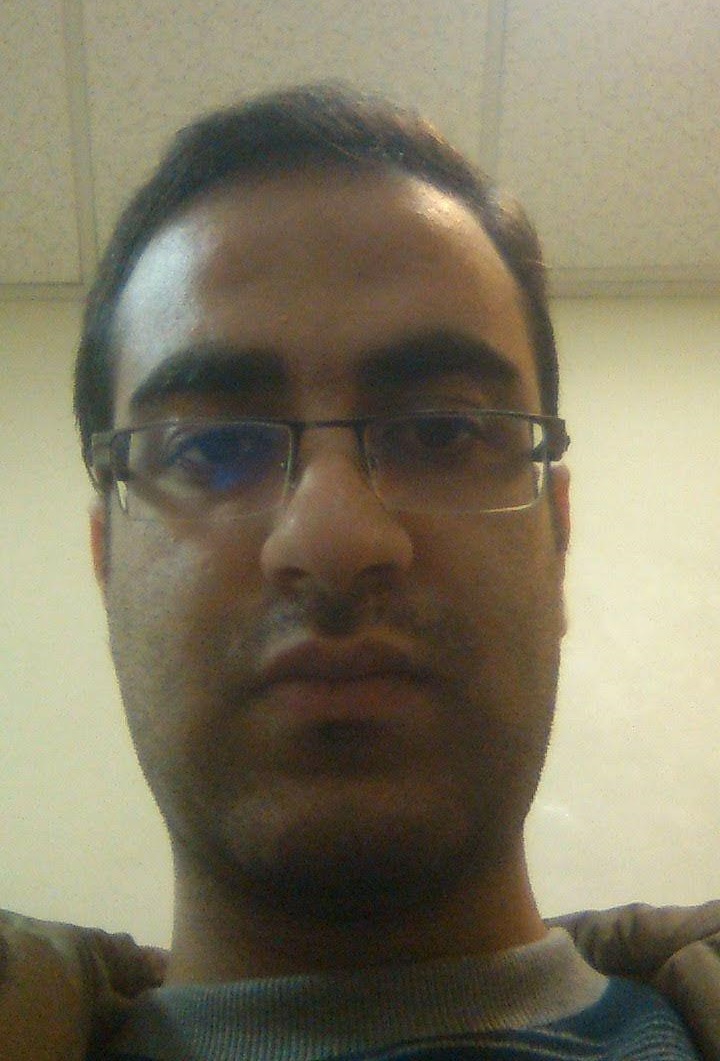}}]{Reza Mirzaeifard} received the B.Sc. degree of computer engineering from Shahid Beheshti University, Tehran, Iran, in 2017. Currently, He is M.Sc. student of Artificial Intelligence program at Computer Engineering Department of Sharif University of Technology. His research interests include machine learning, information theory, tensor decomposition, bioinformatics.
	
\end{IEEEbiography}
\begin{IEEEbiography}[{\includegraphics[width=1in,height=1.25in,clip,keepaspectratio]{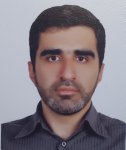}}]{Seyed Abolfazl Motahari}
	is an assistant professor at Computer Engineering Department of Sharif University of Technology (SUT). He received his B.Sc. degree from the Iran University of Science and Technology (IUST), Tehran, in 1999, the M.Sc. degree from Sharif University of Technology, Tehran, in 2001, and the Ph.D. degree from University of Waterloo, Waterloo, Canada, in 2009, all in electrical engineering. From October 2009 to September 2010, he was a Postdoctoral Fellow with the University of Waterloo, Waterloo. From September 2010 to July 2013, he was a Postdoctoral Fellow with the Department of Electrical Engineering and Computer Sciences, University of California at Berkeley. His research interests include multiuser information theory and Bioinformatics. He received several awards including Natural Science and Engineering Research Council of Canada (NSERC) Post-Doctoral Fellowship.
\end{IEEEbiography}
\begin{IEEEbiography}[{\includegraphics[width=1in,height=1.25in,clip,keepaspectratio]{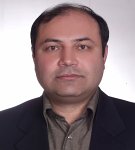}}]{Mohammad-Taghi Manzuri Shalmani}
	received the B.Sc. and M.Sc.
	in electrical engineering from Sharif University of
	Technology (SUT), Iran, in 1984 and 1988, respectively.
	He received the Ph.D. degree in electrical and
	computer engineering from the Vienna University
	of Technology, Austria, in 1995. Currently, he is an
	associate professor in the Computer Engineering
	Department, Sharif University of Technology,
	Tehran, Iran. His main research interests include
	digital signal processing, stochastic modeling, and
	Multi-resolution signal processing.
\end{IEEEbiography}
\vfill




\end{document}